\pdfoutput=1
\documentclass{article}

\usepackage[numbers]{natbib}
\usepackage[nonatbib,final]{neurips_2021}

\usepackage[utf8]{inputenc} 
\usepackage[T1]{fontenc}    
\usepackage{hyperref}       
\usepackage{url}            
\usepackage{booktabs}       
\usepackage{amsfonts}       
\usepackage{nicefrac}       
\usepackage{microtype}      
\usepackage{xcolor}         

\usepackage{multirow}
\usepackage{booktabs}
\usepackage[pdftex]{graphicx}
\usepackage{makecell}
\usepackage{tabu}
\usepackage{collectbox}
\usepackage{scalerel}
\usepackage{qiangstyle}

\title{Conflict-Averse Gradient Descent \\
for Multi-task Learning}

\author{%
  $^\dagger$Bo Liu, $^\dagger$Xingchao Liu, $^\ddagger$Xiaojie Jin, $^{\dagger,\mathsection}$Peter Stone, $^\dagger$Qiang Liu \\
  $^\dagger$The University of Texas at Austin, $^{\mathsection}$Sony AI, $^\ddagger$Bytedance Research\\
  \texttt{\{bliu,xcliu,pstone,lqiang\}@cs.utexas.edu, xjjin0731@gmail.com} \\
}

\newcommand{\fs}[1]{\footnotesize $\pm$#1}

\newcommand{\mysubsection}[1]{\vspace{-3pt}\subsection{#1}\vspace{-3pt}}

\definecolor{myred}{rgb}{0.858, 0.1, 0.1}
\definecolor{myyellow}{rgb}{0.858, 0.6, 0.1}
\definecolor{mydarkred}{rgb}{0.69,0.0,0.1098}

\begin{document}

\maketitle

\begin{abstract}
The goal of multi-task learning is to enable more efficient learning than single task learning by sharing model structures for a diverse set of tasks.
A standard multi-task learning objective is to minimize the average loss across all tasks. While straightforward, using this objective often results in much worse final performance for each task than learning them independently.
A major challenge in optimizing a multi-task model is the \emph{conflicting gradients}, where gradients of different task objectives are not well aligned
so that following the average gradient direction can be detrimental to specific tasks' performance. 
Previous work has proposed several heuristics to manipulate the task gradients for mitigating this problem. But most of them lack convergence guarantee and/or could converge to any Pareto-stationary point.
In this paper, we introduce Conflict-Averse Gradient descent (CAGrad) which minimizes the 
average loss function, while leveraging 
the worst local improvement of individual tasks to regularize the algorithm trajectory.
CAGrad balances the objectives automatically and still provably converges to a minimum over the average loss. It includes the regular gradient descent (GD) and the multiple gradient descent algorithm (MGDA) in the multi-objective optimization (MOO) literature as special cases.
On a series of challenging multi-task supervised learning and reinforcement learning tasks, CAGrad achieves improved performance over prior state-of-the-art multi-objective gradient manipulation methods. Code is available at \url{https://github.com/Cranial-XIX/CAGrad}.
\end{abstract}

\section{Introduction}
Multi-task learning (MTL) refers to learning a single model that can tackle multiple different tasks~\citep{hashimoto2016joint, ruder2017overview, zhang2021survey, vandenhende2021multi}. By sharing parameters across tasks, MTL methods learn more efficiently with an overall smaller model size compared to learning with separate models~\citep{vandenhende2021multi, yang2020multi, misra2016cross}. Moreover, it has been shown that MTL could in principle improve the quality of the learned representation and therefore benefit individual tasks~\citep{swersky2013multi, zamir2018taskonomy, stein2020inadmissibility}. For example, an early MTL result by \cite{caruana1997multitask} demonstrated that training a neural network to recognize doors could be improved by simultaneously training it to recognize doorknobs.

However, learning multiple tasks simultaneously can be a challenging optimization problem because it involves multiple objectives~\citep{vandenhende2021multi}. The most popular MTL objective in practice is the average loss over all tasks. Even when this average loss is exactly the true objective (as opposed to only caring about a single task as in the door/doorknob example), directly optimizing the average loss could lead to undesirable performance, e.g. the optimizer struggles to make progress so the learning performance significantly deteriorates. A known cause of this phenomenon is the \emph{conflicting gradients}~\citep{yu2020gradient}: gradients from different tasks 1) may have varying scales with the largest gradient dominating the update, and 2) may point in different directions so that directly optimizing the average loss can be quite detrimental to a specific task's performance.

To address this problem, previous work either adaptively re-weights the objectives of each task based on heuristics~\citep{chen2018gradnorm, kendall2018multi} or seeks a better update vector~\citep{sener2018multi, yu2020gradient} by manipulating the task gradients. However, existing work often lacks convergence guarantees or only provably converges to any point on the Pareto set of the objectives. This means the final convergence point of these methods may largely depend on the initial model parameters. As a result, it is possible that these methods over-optimize one objective while overlooking the others (See Fig.~\ref{fig:toy}).

Motivated by the limitation of current methods, we introduce Conflict-Averse Gradient descent (CAGrad), which reduces the conflict among gradients and still provably converges to a minimum of the average loss. The idea of CAGrad is simple: it looks for an update vector that maximizes the worst local improvement of any objective in a neighborhood of the average gradient. In this way, CAGrad automatically balances different objectives and smoothly converges to an optimal point of the average loss. Specifically, we show that vanilla gradient descent (GD) and the multiple gradient descent algorithm (MGDA) are special cases of CAGrad (See Sec.~\ref{sec:CAGrad}). We demonstrate that CAGrad can improve over prior state-of-the-art gradient manipulation methods on a series of challenging multi-task supervised, semi-supervised, and reinforcement learning problems.

\begin{figure}[t]
    \centering
    \includegraphics[width=\textwidth]{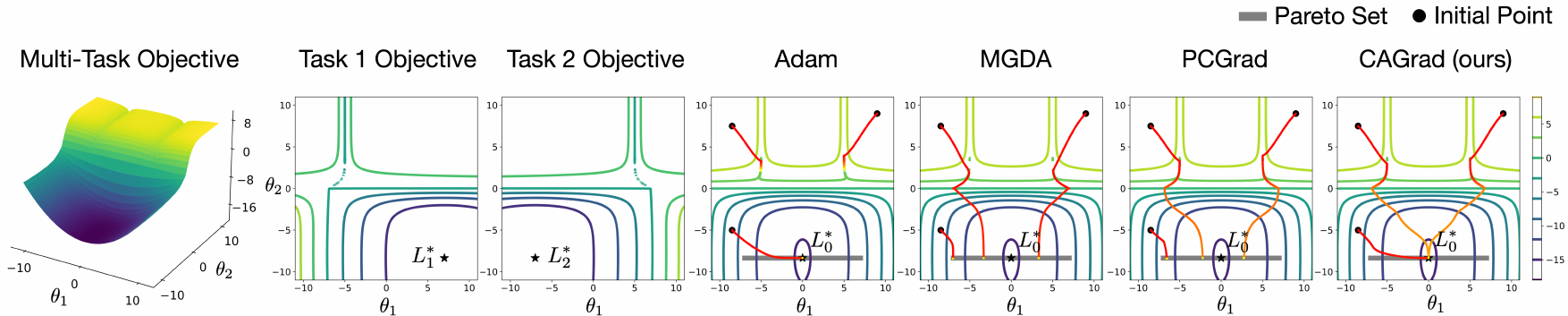}
    \vspace{-15pt}
    \caption{The optimization challenges faced by gradient descent (GD) and existing gradient manipulation methods like the multiple gradient descent algorithm (MGDA)~\citep{desideri2012multiple} and PCGrad~\citep{yu2020gradient}. MGDA, PCGrad and CAGrad are applied on top of the Adam optimizer~\citep{kingma2014adam}. For each methods, we repeat 3 runs of optimization from different initial points (labeled with $\bullet$). Each optimization trajectory is colored from \textcolor{mydarkred}{red} to \textcolor{myyellow}{yellow}. GD with Adam gets stuck on two of the initial points because the gradient of one task overshadows that of the other task, causing the algorithm to jump back and forth between the walls of a steep valley without making progress along the floor of the valley. MGDA and PCGrad stop optimization as soon as they reach the Pareto set.}
    \label{fig:toy}
    \vspace{-5pt}
\end{figure}
\section{Background}
In this section, we first introduce the problem setup of multi-task learning (MTL). Then we analyze the optimization challenge of MTL and discuss the limitation of prior gradient manipulation methods.

\mysubsection{Multi-task Learning and its Challenge}
\label{sec:setting}
In multi-task learning (MTL), we are given $K \geq 2$ different tasks, 
each of which is associated with a loss function $L_i(\theta)$ for 
a shared set of parameters $\theta$.
The goal is to find an optimal $\theta \in \RR^m$ that achieves low losses across all tasks. In practice, a standard objective for MTL is minimizing the average loss over all tasks:
\begin{equation}
\theta^* = \argmin_{\theta \in \RR^m}\left \{  L_0(\theta) \triangleq \frac{1}{K}\sum_{i=1}^K L_i(\theta) \right\}.
\label{eq:mtl-obj}
\end{equation}
Unfortunately, directly optimizing~\eqref{eq:mtl-obj} using gradient descent may significantly compromise the optimization of individual losses in practice. A major source of this phenomenon is known as the conflicting gradients~\citep{yu2020gradient}.

\textbf{Optimization Challenge: Conflicting Gradients}~~~
Denote by $g_i = \nabla L_i(\theta)$ the gradient of task $i$, and $g_0 = \nabla L_0(\theta) = \frac{1}{K} \sum_i^K g_i$ the averaged gradient. With learning rate $\alpha \in \RR^+$, $\theta \leftarrow \theta - \alpha g_0$ is the steepest descent update that appears to be the most natural update to follow when optimizing \eqref{eq:mtl-obj}. However, $g_0$ may conflict with individual gradients, i.e. $\exists~~i,~\langle g_i, g_0 \rangle <0$. When this conflict is large, following $g_0$ will decrease the performance on task $i$. As observed by \cite{yu2020gradient} and illustrated in Fig.~\ref{fig:toy}, when $\theta$ is near a steep ``valley", where a specific task's gradient dominates the update, manipulating the direction and magnitude of $g_0$ often leads to better optimization.

\mysubsection{Prior Attempts and Convergence Issues}
Several methods have been proposed to manipulate the task gradients to form a new update vector and have shown improved performance on MTL. Sener et al. apply the multiple-gradient descent algorithm (MGDA) \citep{desideri2012multiple} for MTL, which directly optimizes towards the Pareto set~\citep{sener2018multi}. Chen et al. dynamically re-weight each $L_i$ using a pre-defined heuristic~\citep{chen2018gradnorm}. More recently, PCGrad identifies conflicting gradients as the motivation behind manipulating the gradients and projects each task gradient to the normal plane of others to reduce the conflict~\citep{yu2020gradient}. While all these methods have shown success at improving the learning performance of MTL,
they manipulate the gradient without respecting the original objective~\eqref{eq:mtl-obj}. Therefore, these methods could in principle converge to any point in the Pareto set (See Fig.~\ref{fig:toy} and Sec.~\ref{sec:convergence}). We provide the detailed algorithms of MGDA and PCGrad in Appendix~\ref{apx:mgda} and~\ref{apx:pcgrad}, and a visualization of the update vector by each method in Fig.~\ref{fig:compare}.
\section{Method}
We introduce our main algorithm, Conflict-Averse Gradient descent in Sec.~\ref{sec:CAGrad}, and then show theoretical analysis in Sec.~\ref{sec:convergence}. 

\mysubsection{Conflict-Averse Gradient Descent}
\label{sec:CAGrad}
Assume we update $\theta$ by $\theta' \gets \theta - \alpha d$, where $\alpha$ is a step size and $d$ an update vector. We want to choose $d$ to decrease not only the average loss $L_0$, but also every individual loss. To do so, we consider the minimum decrease rate across the losses, 
\begin{equation}
    R(\theta, d) 
    = \max_{i \in [K]}\left \{  \frac{1}{\alpha} \left (L_i(\theta - \alpha d) - L_i(\theta) \right)  \right\} 
    \approx -  \min_{i\in[K]} \langle g_i,  d\rangle,
    \label{eq:conflict}
\end{equation}
where we use the first-order Taylor approximation assuming $\alpha$ is small. If $R(\theta, d) <0$, it means that all losses are decreased with the update given a sufficiently small $\alpha$. Therefore, $R(\theta, d)$ can be regarded as a measurement of conflict among objectives.

With the above measurement, our algorithm finds an update vector that minimizes such conflict to mitigate the optimization challenge while still converging to an optimum of the main objective $L_0(\theta)$.
To this end, we introduce Conflict-Averse Gradient descent (CAGrad), which on each optimization step determines the update $d$ by solving the following optimization problem:
\begin{equation}
    \max_{d\in \RR^m}\min_{i\in[K]} \langle g_i, d\rangle~~~~\text{s.t.}~~~~\norm{d - g_0} \leq c\norm{g_0},
    \label{eq:CAGrad}
\end{equation}
Here, $c \in [0, 1)$ is a pre-specified hyper-parameter that controls the convergence rate (See Sec.~\ref{sec:convergence}). 
The optimization problem ~\eqref{eq:CAGrad} looks for the best update vector within a local ball centered at the averaged gradient $g_0$, which also minimizes the conflict in losses measured by \eqref{eq:conflict}. Since we focus on MTL and choose the average loss as the main objective, $g_0$ is the average gradient. However, CAGrad also applies when $g_0$ is the gradient of some other user-specified objective.
We leave exploring this possibility as a future direction. 

\paragraph{Dual Objective} The optimization problem~\eqref{eq:CAGrad} involves decision variable $d$ that has the same dimension as the number of parameters in $\theta$, which could be millions for a deep neural network. 
It is not practical to directly solve for $d$ on every optimization step.
However, the dual problem of Eq.~\eqref{eq:CAGrad}, as we will derive in the following, only involves solving for a decision variable $w \in \mathbb{R}^K$, which can be efficiently found using standard optimization libraries~\citep{diamond2016cvxpy}. Specifically, first note that $\min_i \langle g_i, d\rangle = \min_{w \in \mathcal W} \langle \sum_i w_i g_i, d\rangle$, where $w = (w_1,\dots, w_K) \in \RR^K$ and $\mathcal W$ denotes the probability simplex, i.e. $\mathcal W = \{w \colon \sum_i w_i= 1~\text{and}~ w_i\geq 0\}$. Denote $g_w = \sum_i w_i g_i$ and $\phi = c^2 \norm{g_0}^2$.
The Lagrangian of the objective in Eq.~\eqref{eq:CAGrad} is
$$
\max_{d\in \mathbb{R}^m} \min_{\lambda \geq 0, w\in\mathcal{W}}  g_w ^\top d - \lambda (\norm{g_0 - d}^2 -  \phi)/2. 
$$
Since the objective for $d$ is concave with linear constraints, by switching the $\min$ and $\max$, we reach the dual form without changing the solution 
by Slater's condition:  
$$
 \min_{\lambda \geq 0, w \in \mathcal{W}}  \max_{d\in \mathbb{R}^m} g_w ^\top d - \lambda \norm{g_0 - d}^2/2 + \lambda  \phi /2.
$$
We end up with the following optimization problem w.r.t. $w$ after several steps of calculus, 
$$
 w^* = \argmin_{w \in \mathcal{W}}  g_w ^\top g_0    +  \sqrt{\phi} \norm{g_w},
$$
where the optimal $\lambda^* = \norm{g_{w^*}}/\phi^{1/2}$ and the optimal update $d^* = g_0 + g_{w^*}/\lambda^*$. The detailed derivation is provided in Appendix~\ref{apx:cagrad} and the entire CAGrad algorithm is summarized in Alg.~\ref{alg:caged}. 
The dimension of $w$ equals to the number of objectives $K$, which usually ranges from $2$ to tens and is much smaller than the number of parameters in a neural network. Therefore, in practice, we solve the dual objective to perform the update of CAGrad.
\begin{algorithm*}[t]
    \caption{Conflict-averse Gradient Descent (CAGrad) for Multi-task Learning}
    \begin{algorithmic} 
        \STATE \textbf{Input}: Initial model parameter vector $\theta_0$, differentiable loss functions $\{L_i\}_{i=1}^K$, a constant $c\in[0,1)$ and learning rate $\alpha \in \RR^+$.
        \REPEAT
            \STATE At the $t$-th optimization step, define $g_0 = \frac{1}{K}\sum_{i=1}^K\nabla L_i(\theta_{t-1})$ and $\phi = c^2\norm{g_0}^2$.
            \STATE Solve $$\min_{w \in \mathcal W} F(w) :=  g_w ^\top g_0 + \sqrt{\phi} \norm{g_w},~\text{where}~g_w = \sum_{i=1}^K w_i\nabla L_i(\theta_{t-1}).$$
            \STATE Update $\theta_t = \theta_{t-1} -\alpha \left (g_0 + \frac{\phi^{1/2}}{\norm{g_w}} g_w \right ).$
        \UNTIL{convergence}
    \end{algorithmic}
    \label{alg:caged}
\end{algorithm*}

\paragraph{Remark} 
 In Alg.~\ref{alg:caged}, when $c=0$, CAGrad recovers
 the typical gradient descent with $d = g_0$. 
 On the other hand, when $c \rightarrow \infty$, 
 then minimizing $F(w)$ is equivalent to $\min_w \norm{g_w}$. This coincides with the multiple gradient descent algorithm (MGDA) \citep{desideri2012multiple}, 
 which uses the minimum norm vector in the convex hull of the individual gradients as the update direction 
 (see Fig.~\ref{fig:compare}; second column). 
MGDA is a gradient-based multi-objective optimization 
designed to converge to an arbitrary point on the Pareto set, that is, it leaves all the points on the Pareto set as fixed points (and hence can not control which specific point it will converge to). It is different from our method which targets to minimize $L_0$ while using  gradient conflict to regularize the optimization trajectory. As we will analyze in the following section, to guarantee that CAGrad converges to an optimum of $L_0(\theta)$, we have to ensure $0 \leq c < 1$.

\mysubsection{Convergence Analysis}
\label{sec:convergence}
In this section we first formally introduce the related Pareto concepts and then analyze CAGrad's convergence property. Particularly, in Alg.~\ref{alg:caged}, when $c<1$, CAGrad is guaranteed to converge to a minimum point of the average loss $L_0$.

\textbf{Pareto Concepts}~~~Unlike single task learning where any two parameter vectors $\theta_1$ and $\theta_2$ can be ordered in the sense that either 
$L(\theta_1) \leq L(\theta_2)$  
or $L(\theta_1) \geq L(\theta_2)$ holds, 
MTL could have two parameter vectors where one performs better for task $i$ and the other performs better for task $j\neq i$. To this end,
we need the notion of Pareto-optimality~\citep{hochman1969pareto}.

\begin{mydef}[Pareto optimal and stationary points]
\label{def:pareto-optimality}
Let $\vv L(\theta) = \{L_i(\theta) \colon i \in [K]\}$ 
be a set of differentiable loss functions from $\RR^m$ to $\RR$.
For two points $\theta, \theta' \in \RR^m$,  
we say that $\theta$ is Pareto dominated by $\theta'$, denoted by $\vv L(\theta') \prec \vv L(\theta)$, if $L_i(\theta') \leq L_i(\theta)$ for all $i \in [K]$ and $\vv L(\theta') \neq \vv L(\theta)$. 
A point $\theta \in \RR^m$ is said to be Pareto-optimal if there   exists no    $\theta' \in \RR^m$ such that $\vv L(\theta') \prec \vv L(\theta)$. 
The set of all Pareto-optimal points is called the Pareto set. 
A point $\theta$ is called Pareto-stationary if  we have $\min_{w\in \mathcal W}  \norm{g_w(\theta)} = 0$, where $g_w(\theta) = \sum_{i=1}^K w_i \dd L_i(\theta),$ and $\mathcal W$  is the probability simplex on $[K].$
\end{mydef}

\begin{figure}[t]
    \centering
    \includegraphics[width=\textwidth]{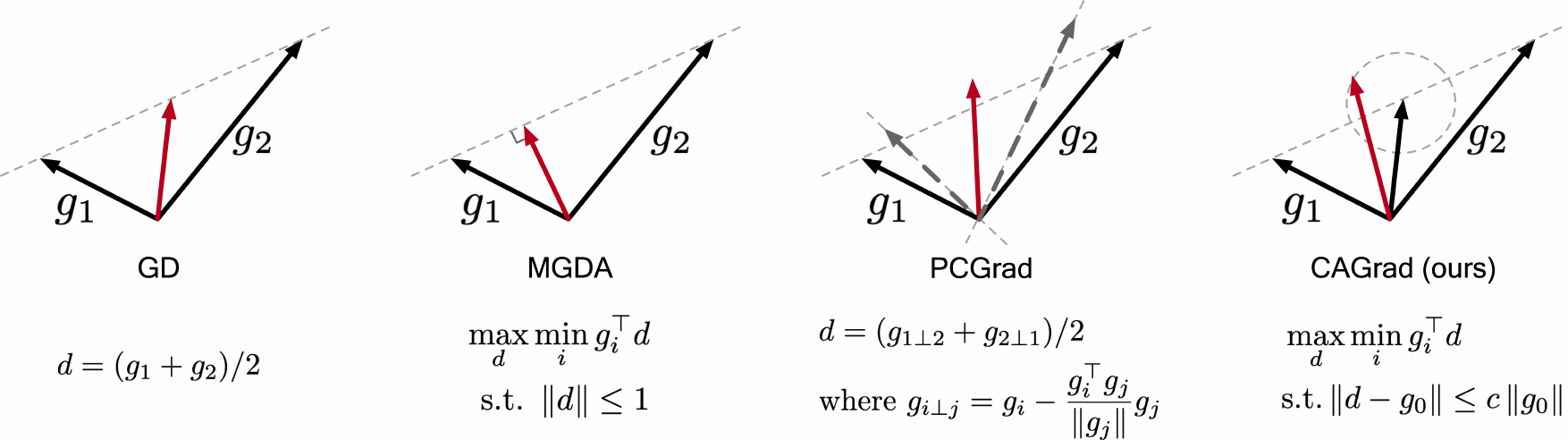}
    \vspace{-15pt}
    \caption{The combined update vector $d$ (in \textcolor{mydarkred}{red}) of a two-task learning problem with gradient descent (GD), multiple gradient descent algorithm (MGDA), PCGrad and Conflict-Averse Gradient descent (CAGrad). The two task-specific gradients are labeled $g_1$ and $g_2$. MGDA's objective is given in its primal form (See Appendix~\ref{apx:mgda}). For PCGrad, each gradient is first projected onto the normal plane of the other (the dashed arrows). Then the final update vector is the average of the two projected gradients. CAGrad finds the best update vector within a ball around the average gradient that maximizes the worse local improvement between task 1 and task 2.}
    \label{fig:compare}
\end{figure}

Similar to the case of single-objective differentiable optimization, 
a local Pareto optimal point $\theta$ must be Pareto stationary
(see e.g., \cite{desideri2012multiple}). 
%

\begin{thm}[Convergence of CAGrad]
\label{thm:cagrad}
Assume the individual loss functions $L_0, L_1,\ldots, L_K$ are differentiable on $\RR^m$ and their gradients $\nabla L_i(\theta)$ are all $H$-Lipschitz, i.e. $\norm{\nabla L_i(x) - \nabla L_i(y)} \leq H\norm{x-y}$ for $i=0,1,\ldots, K$ where $0 \leq H \leq \infty$.  Assume $L_0^* = \inf_{\theta\in\RR^m} L_0(\theta) > -\infty$. 

With a fixed step size $\alpha$ satisfying $0 < \alpha \leq 1/H$, 
we have for the CAGrad in Alg.~\ref{alg:caged}: 

1) For any $c \geq 1$, all the fixed points of CAGrad are  Pareto-stationary points of $(L_0,L_1, \ldots, L_K)$. 

2) In particular, if we take $0 \leq c < 1$, then CAGrad satisfies 
$$
\sum_{t=0}^T \norm{\dd L_0(\theta_t)}^2
\leq  \frac{2(L_0(\theta_0) - L_0^*)}{{\alpha} (1-c^2)}.
$$
\end{thm}
This means that the algorithm converges to a stationary point of $\dd L_0$ if we take $0\leq c < 1$. 
The proof is in Appendix~\ref{apx:cagrad}. As we discuss earlier, 
unlike our method, MGDA is designed to converge to an arbitrary point on the Pareto set, without explicit control of which point it will converges to. 
Another algorithm with similar property is PCGrad \citep{yu2020gradient}, which is a gradient-based algorithm that mitigates the conflicting gradients problem by removing the conflicting components of each gradient with respect to the other gradients before averaging them to form the final update; see Fig.~\ref{fig:compare}, third column for the illustration.
Similar to MGDA, 
as shown in \citep{yu2020gradient}, 
PCGrad also converges to an arbitrary Pareto point without explicit control of which point it will arrive at. 

\mysubsection{Practical Speedup} 
\label{sec:speedup}
A typical drawback of methods that manipulate gradients is the computation overhead. For computing the optimal update vector, a method usually requires $K$ back-propagations to find all individual gradients $g_i$, in addition to the time required for optimization. This can be prohibitive for the scenario with many tasks. To this end, we propose to only sample a subset of tasks $S \subseteq [K]$, compute their corresponding gradients $\{g_i \mid i\in S\}$ and the averaged gradient $g_0$. Then we optimize $d$ in:
\begin{equation}
    \begin{split}
    \max_{d \in \mathbb{R}^m} \min \bigg(\langle \frac{K g_0 - \sum_{i \in S} g_i}{K-|S|}, d\rangle,~~\min_{i\in S} \langle g_i, d\rangle \bigg)~~~\text{s.t.}~~ \norm{d-g_0} \leq c\norm{g_0}
    \end{split}
    \label{eq:CAGrad-fast}
\end{equation}
\textbf{Remark} Note that the convergence guarantee in Thm.~\ref{thm:cagrad} still holds for Eq.~\ref{eq:CAGrad-fast} as the constraint does not change (See Appendix~\ref{apx:cagrad}). The time complexity is $\mathcal{O}((|S|N+T)$, where $N$ denotes the time for one pass of back-propagation and $T$ denotes the optimization time. For few-task learning ($K < 10$), usually $T\ll N$. When $S = [K]$, we recover the full CAGrad algorithm.

\section{Related Work}

\textbf{Multi-task Learning}
~~Due to its benefit with regards to data and computational efficiency, multi-task learning (MTL) has broad applications in vision, language, and robotics~\citep{hashimoto2016joint, ruder2017overview, liu2020certified, zhang2021survey, vandenhende2021multi}. 
A number of MTL-friendly architectures have been proposed using task-specific modules~\citep{misra2016cross, hashimoto2016joint}, attention-based mechanisms~\citep{liu2019end} or activating different paths along the deep networks to tackle MTL~\citep{rosenbaum2017routing, yang2020multi}. Apart from designing new architectures, another branch of methods focus on decomposing a large problem into smaller local problems that could be quickly learned by smaller models~\citep{rusu2015policy,parisotto2015actor,teh2017distral,ghosh2017divide}. Then a unified policy is learned from the smaller models using knowledge distillation~\citep{hinton2015distilling}.

\textbf{MTL Optimization} ~~In this work, we focus on the optimization challenge of MTL~\citep{vandenhende2021multi}.  Gradient manipulation methods are designed specifically to balance the learning of each task. The simplest form of gradient manipulation is to re-weight the task losses based on specific criteria, e.g., uncertainty~\citep{kendall2018multi}, gradient norm~\citep{chen2018gradnorm}, or difficulty~\citep{guo2018dynamic}. 
These methods are mostly heuristics and their performance can be unstable. Recently, two methods~\citep{sener2018multi, yu2020gradient} that manipulate the gradients to find a better local update vector have become popular. Sener et al~\citep{sener2018multi} view MTL as a multi-objective optimization problem, and use multiple gradient descent algorithm for optimization. PCGrad~\citep{yu2020gradient} identifies a major optimization challenge for MTL, the conflicting gradients, and proposes to project each task gradient to the normal plane of other task gradients before combining them together to form the final update vector. Though yielding good empirical performance, both methods can only guarantee convergence to a Pareto-stationary point, but not knowing where it exactly converges to. More recently, GradDrop~\citep{chen2020just} randomly drops out task gradients based on how much they conflict. IMTL-G~\citep{liu2020towards} seeks an update vector that has equal projections on each task gradient. RotoGrad~\cite{javaloy2021rotograd} separately scales and rotates task gradients to mitigate optimization conflict.

Our method, CAGrad, also manipulates the gradient to find a better optimization trajectory. 
Like other MTL optimization techniques, CAGrad is model-agnostic. However, unlike prior methods, CAGrad converges to the optimal point in theory and achieves better empirical performance on both toy multi-objective optimization tasks and real-world applications.
\section{Experiment}
\label{sec:experiment}
We conduct experiments to answer the following questions: 

\noindent\textbf{Question (1)}  Do CAGrad, MGDA and PCGrad behave consistently with their theoretical properties in practice? (yes)

\noindent\textbf{Question (2)}   
Does CAGrad recover GD and MGDA when varying the constant $c$? (yes)

\noindent\textbf{Question  (3)}  How does CAGrad perform in both performance and computational efficiency compared to prior state-of-the-art methods, on challenging multi-task learning problems under the supervised, semi-supervised and reinforcement learning settings? (CAGrad improves over prior state-of-the-art methods under all settings)

\mysubsection{Convergence and Ablation over c}
\label{sec:toy_experiment}
To answer questions \textbf{(1)} and \textbf{(2)}, we create a toy optimization example to evaluate the convergence of CAGrad 
compared to MGDA and PCGrad. 
On the same toy example, 
we ablate over the constant $c$
and show that CAGrad recovers GD and MGDA
with proper $c$ values. 
Next, to test CAGrad on  more complicated neural  models, 
we perform the same set of experiments on the  Multi-Fashion+MNIST benchmark~\citep{lin2019pareto} 
with a shrinked LeNet architecture~\citep{lecun1998gradient} 
(in which each layer has a reduced number of neurons compared to the original LeNet). Please refer to Appendix~\ref{apx:exp} for more details.

For the toy optimization example, we modify the toy example used by Yu et al.~\cite{yu2020gradient} and consider $\theta=(\theta_1, \theta_2) \in \RR^2$ with the following individual loss functions:
\begin{flalign*}
    L_1(\theta) &= c_1(\theta)f_1(\theta) + c_2(\theta)g_1(\theta)~~\text{and}~~L_2(\theta) = c_1(\theta)f_2(\theta) + c_2(\theta)g_2(\theta),~\text{where}\\
    f_1(\theta) &= \log{\big(\max(|0.5(-\theta_1-7)-\tanh{(-\theta_2)}|,~~0.000005)\big)} + 6, \\
    f_2(\theta) &= \log{\big(\max(|0.5(-\theta_1+3)-\tanh{(-\theta_2)}+2|,~~0.000005)\big)} + 6, \\
    g_1(\theta) &= \big((-\theta_1+7)^2 + 0.1*(-\theta_2-8)^2\big)/10-20, \\
    g_2(\theta) &= \big((-\theta_1-7)^2 + 0.1*(-\theta_2-8)^2)\big/10-20, \\
    c_1(\theta) &= \max(\tanh{(0.5*\theta_2)},~0)~~\text{and}~~c_2(\theta) = \max(\tanh{(-0.5*\theta_2)},~0).
\end{flalign*}
The average loss $L_0$ and individual losses $L_1$ and $L_2$ are shown in Fig.~\ref{fig:toy}. We then pick 5 initial parameter vectors $\theta_\text{init}\in\{(-8.5, 7.5), (-8.5, 5), (0,0), (9,9), (10, -8)\}$ and plot the corresponding optimization trajectories with different methods in Fig.~\ref{fig:toy_exp}.
\begin{figure}[t]
    \centering
    \includegraphics[width=\textwidth]{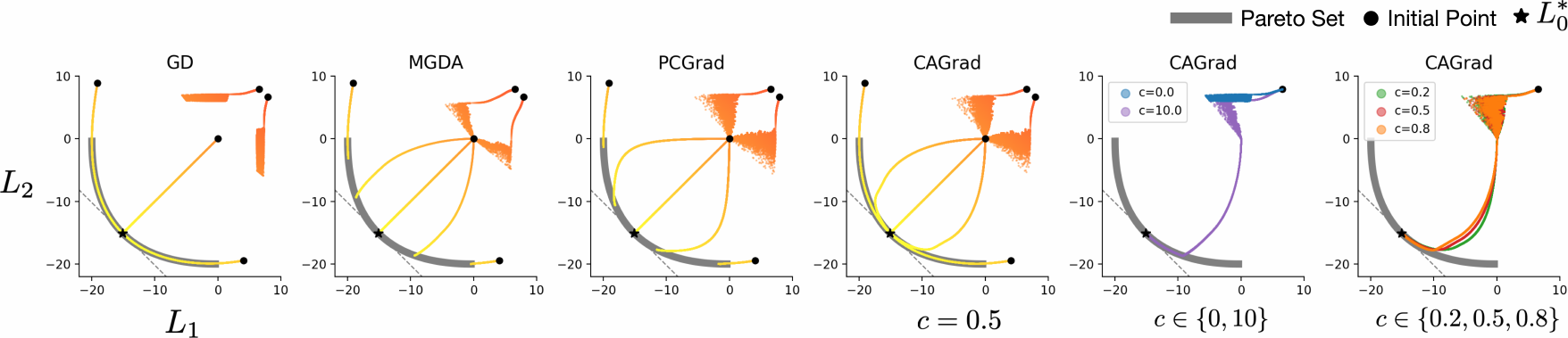}
    \vspace{-15pt}
    \caption{The left four plots are 5 runs of each algorithms from 5 different initial parameter vectors, where trajectories are colored from \textcolor{mydarkred}{red} to \textcolor{myyellow}{yellow}. The right two plots are CAGrad's results with a varying $c \in \{0,0.2,0.5,0.8,10\}$.}
    \label{fig:toy_exp}
    \vspace{-10pt}
\end{figure}
As shown in Fig.~\ref{fig:toy_exp}, GD gets stuck in 2 out of the 5 runs while other methods all converge to the Pareto set. MGDA and PCGrad converge to different Pareto-stationary points depending on $\theta_\text{init}$. CAGrad with $c=0$ recovers GD and CAGrad with $c=10$ approximates MGDA well (in theory it requires $c \rightarrow \infty$ to exactly recover MGDA).

Next, we apply the same set of experiments on the multi-task classification benchmark Multi-Fashion+MNIST~\citep{lin2019pareto}. This benchmark consists of images that are generated by overlaying an image from FashionMNIST dataset~\citep{xiao2017fashion} on top of another image from MNIST dataset~\citep{deng2012mnist}. The two images are positioned on the top-left and bottom-right separately. We consider a shrinked LeNet as our model, and train it with Adam~\citep{kingma2014adam} optimizer with a $0.001$ learning rate for 50 epochs using a batch size of 256. Due to the highly non-convex nature of the neural network, we are not able to visualize the entire Pareto set. But we provide the final training losses of different methods over three independent runs in Fig.~\ref{fig:mnist}. As shown, CAGrad achieves the lowest average loss with $c=0.2$. In addition, PCGrad and MGDA focus on optimizing task 1 and task 2 separately. Lastly, CAGrad with $c=0$ and $c=10$ roughly recovers the final performance of GD and MGDA. By increasing $c$, the model performance shifts from more GD-like to more MGDA-like, though due to the non-convex nature of neural networks, CAGrad with $0\leq c < 1$ does not necessarily converge to the exact same point.

\begin{figure}[ht]
    \centering
    \includegraphics[width=\textwidth]{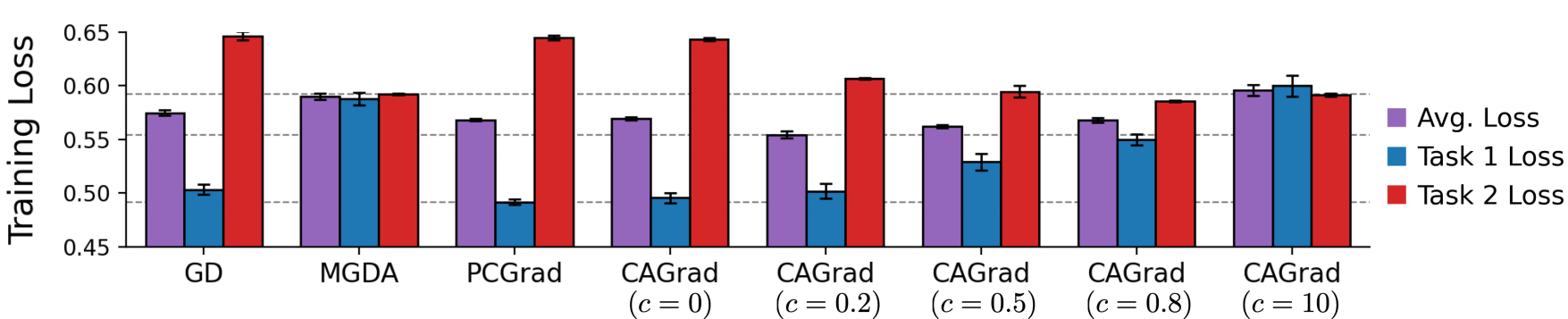}
    \vspace{-15pt}
    \caption{The average and individual training losses on the Fashion-and-MNIST benchmark by running GD, MGDA, PCGrad and CAGrad with different $c$ values. GD gets stuck at the steep valley (the area with a cloud of dots), which other methods can pass. MGDA and PCGrad converge randomly on the Pareto set.}
    \label{fig:mnist}
\end{figure}
\vspace{-4pt}
\mysubsection{Multi-task Supervised Learning}
To answer question \textbf{(3)} in the supervised learning setting, we follow the experiment setup from Yu et al.~\cite{yu2020gradient} and consider the NYU-v2 and CityScapes vision datasets. NYU-v2 contains 3 tasks: 13-class semantic segmentation, depth estimation, and surface normal prediction. CityScapes similarly contains 2 tasks: 7-class semantic segmentation and depth estimation. Here, we follow ~\cite{yu2020gradient} and combine CAGrad with a state-of-the-art MTL method MTAN~\citep{liu2019end}, which applies attention mechanism on top of the SegNet architecture~\citep{badrinarayanan2017segnet}. We compare CAGrad with PCGrad, vanilla MTAN and Cross-Stitch~\citep{misra2016cross}, which is another MTL method that modifies the network architecture. MTAN originally experiments with equal loss weighting and two other dynamic loss weighting heuristics~\citep{kendall2018multi,chen2018gradnorm}. For a fair comparison, all methods are applied under the equal weighting scheme and we use the same training setup from~\citep{chen2018gradnorm}. We search $c \in \{0.1, 0.2, \dots 0.9\}$ with the best average training loss for CAGrad on both datasets ($0.4$ for NYU-v2 and $0.2$ for Cityscapes). We perform a two-tailed, Student's $t$-test under \emph{equal sample sizes, unequal variance} setup and mark the results that are significant with an $*$. Following Maninis et al.\cite{maninis2019attentive}, we also compute the average per-task performance drop of method $m$ with respect to the single-tasking baseline $b$: $\Delta m = \frac{1}{K}\sum_{i=1}^K (-1)^{l_i} (M_{m,i} - M_{b,i})/M_{b,i}$ where $l_i = 1$ if a higher value is better for a criterion $M_i$ on task $i$ and $0$ otherwise. The single-tasking baseline (independent) refers to training individual tasks with a vanilla SegNet. Results are shown in Tab.~\ref{tab:nyu-v2} and Tab.~\ref{tab:cityscapes}.

\begin{table*}[ht]
    \centering
    \resizebox{\textwidth}{!}{%
    \begin{tabular}{ccrrrrrrrrrc}
    \toprule
     & &  \multicolumn{2}{c}{Segmentation} & \multicolumn{2}{c}{Depth} & \multicolumn{5}{c}{Surface Normal} & \\
    \cmidrule(lr){3-4}\cmidrule(lr){5-6}\cmidrule(lr){7-11}
     \#P. & Method &  \multicolumn{2}{c}{(Higher Better)} &  \multicolumn{2}{c}{(Lower Better)}  & \multicolumn{2}{c}{\makecell{Angle Distance\\(Lower Better)}} & \multicolumn{3}{c}{\makecell{Within $t^\circ$\\(Higher Better)}} & $\Delta m \% \downarrow$\\
     & & mIoU  & Pix Acc  & Abs Err & Rel Err & Mean & Median & 11.25 & 22.5 & 30 & \\
     \midrule
     3 & Independent & 38.30 & 63.76 & 0.6754 & 0.2780 & 25.01 & 19.21 & 30.14 & 57.20 & 69.15 & \\
     \midrule
     $\approx$3 & Cross-Stitch~\citep{misra2016cross} & 37.42 & 63.51 & 0.5487 & \textbf{0.2188} & ${}^*$28.85 & ${}^*$24.52 & ${}^*$22.75 & ${}^*$46.58 & ${}^*$59.56 & 6.96\\
    [0.15cm]
     $1.77$ & MTAN~\citep{liu2019end} & 39.29 & 65.33 & 0.5493 & 0.2263 & ${}^*$28.15 & ${}^*$23.96 & ${}^*$22.09 & ${}^*$47.50 & ${}^*$61.08 & 5.59\\
    [0.15cm]
     $1.77$ & MGDA~\citep{sener2018multi} & ${}^*$30.47 & ${}^*$59.90 & ${}^*$0.6070 & ${}^*$0.2555 & \textbf{24.88} & \textbf{19.45} & \textbf{29.18} & \textbf{56.88} & \textbf{69.36} & 1.38\\
    [0.15cm]
     $1.77$ & PCGrad~\citep{yu2020gradient} & 38.06 & 64.64 & 0.5550 & 0.2325 & ${}^*$27.41 & ${}^*$22.80 & ${}^*$23.86 & ${}^*$49.83 & ${}^*$63.14 & 3.97\\
    [0.15cm]
     $1.77$ & GradDrop~\citep{chen2020just} & 39.39 & 65.12 & \textbf{0.5455} & 0.2279 & ${}^*$27.48 & ${}^*$22.96 & ${}^*$23.38 & ${}^*$49.44 & ${}^*$62.87 & 3.58\\
    [0.15cm]
     $1.77$ & CAGrad (ours) & \textbf{39.79} & \textbf{65.49} & 0.5486 & 0.2250 & 26.31 & 21.58 & 25.61 & 52.36 & 65.58 & \textbf{0.20}\\
    \bottomrule 
    \end{tabular}
    }
    \caption{Multi-task learning results on NYU-v2 dataset. \#P denotes the relative model size compared to the vanilla SegNet. Each experiment is repeated over 3 random seeds and the mean is reported. The best average result among all multi-task methods is marked in bold. MGDA, PCGrad, GradDrop and CAGrad are applied on the MTAN backbone. CAGrad has statistically significant improvement over baselines methods with an $*$, tested with a $p$-value of 0.1.}
    \label{tab:nyu-v2}
\end{table*}

\begin{table*}[ht]
    \centering
    \resizebox{0.65\textwidth}{!}{%
    \begin{tabular}{ccrrrrc}
    \toprule
     & & \multicolumn{2}{c}{Segmentation} & \multicolumn{2}{c}{Depth} & \\
    \cmidrule(lr){3-4}\cmidrule(lr){5-6}
     \#P. & Method &  \multicolumn{2}{c}{(Higher Better)} &  \multicolumn{2}{c}{(Lower Better)} & $\Delta m \% \downarrow$ \\
     & & mIoU  & Pix Acc  & Abs Err & Rel Err & \\
     \midrule
    2 & Independent & 74.01 & 93.16 & 0.0125 & 27.77 & \\
     \midrule
    $\approx$3 & Cross-Stitch~\citep{misra2016cross} & ${}^*$73.08 & ${}^*$92.79 & ${}^*$0.0165 & ${}^*$118.5 & 90.02\\
    [0.15cm]
    $1.77$ & MTAN~\citep{liu2019end}  & 75.18 & 93.49 & ${}^*$0.0155 & ${}^*$46.77 & 22.60\\
    [0.15cm]
    $1.77$ & MGDA~\citep{sener2018multi}  & ${}^*$68.84 & ${}^*$91.54 & 0.0309 & \textbf{33.50} & 44.14\\
    [0.15cm]
    $1.77$ & PCGrad~\citep{yu2020gradient} & 75.13 & 93.48 & 0.0154 & 42.07  & 18.29\\
    [0.15cm]
    $1.77$ & GradDrop~\citep{chen2020just}  & \textbf{75.27} & \textbf{93.53} & ${}^*$0.0157 & ${}^*$47.54 & 23.73\\
    [0.15cm]
    $1.77$ & CAGrad (ours) & 75.16 & 93.48 & \textbf{0.0141} & 37.60 & \textbf{11.64} \\
    \bottomrule
    \end{tabular}
    }
    \caption{Multi-task learning results on CityScapes Challenge. Each experiment is repeated over 3 random seeds and the mean is reported. The best average result among all multi-task methods is marked in bold. PCGrad and CAGrad are applied on the MTAN backbone. CAGrad has statistically significant improvement over baselines methods with an $*$, tested with a $p$-value of 0.1.}
    \label{tab:cityscapes}
    \vspace{-5pt}
\end{table*}
Given the single task performance, CAGrad performs better on the task that is overlooked by other methods (Surface Normal in NYU-v2 and Depth in CityScapes) and matches other methods' performance on the rest of the tasks. We also provide the final test losses and the per-epoch training time of each method in Fig.~\ref{fig:nyu-comparison} in Appendix~\ref{apx:exp_sl}.

\mysubsection{Multi-task Reinforcement Learning}
To answer question \textbf{(3)} in the reinforcement learning (RL) setting, we apply CAGrad on the MT10 and MT50 benchmarks from the Meta-World environment~\citep{yu2020meta}. In particular, MT10 and MT50 contains 10 and 50 robot manipulation tasks. Following \cite{sodhani2021multi}, we use Soft Actor-Critic (SAC)~\citep{haarnoja2018soft} as the underlying RL training algorithm. We compare against Multi-task SAC (SAC with a shared model), Multi-headed SAC (SAC with a shared backbone and task-specific head), Multi-task SAC + Task Encoder (SAC with a shared model and the input includes a task embedding)~\citep{yu2020meta} and PCGrad~\citep{yu2020gradient}. We also compare with Soft Modularization~\citep{yang2020multi} that routes different modules in a shared model to form different policies. Lastly, we also include a recent method (CARE) that considers language metadata and uses a mixture of expert encoder for MTL. We follow the same experiment setup from~\citep{sodhani2021multi}. The results are shown in Tab.~\ref{tab:metaworld}. CAGrad outperforms all baselines except for CARE which benefits from extra information from the metadata. We also apply the practical speedup in Sec.~\ref{sec:speedup} and sub-sample 4 and 8 tasks for MT10 and MT50 (CAGrad-Fast). CAGrad-fast achieves comparable performance against the state-of-the-art method while achieving a $2$x (MT10) and $5$x (MT50) speedup over PCGrad. We provide a visualization of tasks from MT10 and MT50, and the comparison of computational efficiency in Appendix~\ref{apx:exp_rl}.

\begin{table}[ht]
    \vspace{-10pt}
    \centering
    \resizebox{0.8\textwidth}{!}{%
    \begin{tabular}{lcc}
    \toprule
    & Metaworld MT10 & Metaworld MT50 \\
    \cmidrule(lr){2-2}\cmidrule(lr){3-3}
    Method & success & success\\
    & (mean $\pm$ stderr) & (mean $\pm$ stderr) \\
    \midrule
    Multi-task SAC~\citep{yu2020meta}                 & 0.49 \fs{0.073}                    & 0.36 \fs{0.013}\\
    Multi-task SAC + Task Encoder~\citep{yu2020meta}  & 0.54 \fs{0.047}                    & 0.40 \fs{0.024}\\
    Multi-headed SAC~\citep{yu2020meta}               & 0.61 \fs{0.036}                    & 0.45 \fs{0.064}\\
    PCGrad~\citep{yu2020gradient}                     & 0.72 \fs{0.022}                    & 0.50 \fs{0.017}\\
    Soft Modularization~\citep{yang2020multi}         & 0.73 \fs{0.043}                    & 0.50 \fs{0.035}\\
    CAGrad (ours)                                     & \textbf{0.83} \fs{0.045}           & \textbf{0.52} \fs{0.023} \\
    CAGrad-Fast (ours)                                & 0.82 \fs{0.039}                    & 0.50 \fs{0.016} \\
    \midrule
    CARE~\citep{sodhani2021multi}                     & 0.84 \fs{0.051}                    & 0.54 \fs{0.031}\\
    One SAC agent per task (upper bound)              & 0.90 \fs{0.032}                    & 0.74 \fs{0.041}\\
    \bottomrule
    \end{tabular}
    }
    \vspace{5pt}
    \caption{Multi-task reinforcement learning results on the Metaworld benchmarks. Results are averaged over 10 independent runs and the best result is marked in bold.}
    \label{tab:metaworld}
\end{table}
\mysubsection{Semi-supervised Learning with Auxiliary Tasks}
Training with auxiliary tasks to improve the performance of a main task is another popular application of MTL. Here, we take semi-supervised learning as an instance. We combine different optimization algorithms with Auxiliary Task Reweighting for Minimum-data Learning (ARML)~\citep{shi2020auxiliary}, a state-of-the-art semi-supervised learning algorithm. The loss function is composed of the main task and two auxiliary tasks:

\begin{equation}
L_0 = L_{CE}(\theta; D_{l}) + w_1 L_{aux}^1(\theta; D_{u}) + w_2 L_{aux}^2 (\theta; D_u),
\end{equation}

where $L_{CE}$ is the main cross-entropy classification loss on the labeled dataset $D_l$, and $L_{aux}^1, L_{aux}^2$ are auxiliary unsupervised learning losses on the unlabeled dataset $D_u$. We use the same $w_1$ and $w_2$ from ARML, and use the CIFAR10 dataset~\citep{krizhevsky2009learning}, which contains 50,000 training images and 10,000 test images. 10\% of the training images is held out as  the validation set. We test PCGrad, MGDA and CAGrad with 500, 1000 and 2000 labeled images. The rest of the training set is used for auxiliary tasks. For all the methods, we use the same labeled dataset, the same learning rate and train them for 200 epochs with the Adam~\citep{kingma2014adam} optimizer. Please refer to Appendix~\ref{apx:exp_ssl} for more experimental details. Results are shown in Tab.~\ref{tab:ssl}. With all the different number of labels, CAGrad yields the best averaged test accuracy. We observed that MGDA performs much worse than the ARML baseline, because it significantly overlooks the main classification task. We also compare different gradient manipulation methods on the same task with GradNorm~\citep{chen2018gradnorm}, which dynamically adjusts $w_1$ and $w_2$ during training. The results and conclusions are similar to those for ARML. 
\begin{table}[ht]
    \centering
    \begin{tabular}{lccc}
    \toprule
    Method & 500 labels & 1000 labels & 2000 labels \\
    \midrule
    ARML~\citep{shi2020auxiliary} & 67.05 \fs{0.16} & 73.22 \fs{0.26} & 81.35 \fs{0.36} \\
    ARML + PCGrad~\citep{yu2020gradient} & 67.49 \fs{0.64} & 73.23 \fs{0.62} & 81.91 \fs{0.19} \\
    ARML + MGDA~\cite{sener2018multi} & 49.27 \fs{0.68} & 60.11 \fs{2.35} & 60.78 \fs{0.17} \\
    ARML + CAGrad (Ours) & \textbf{68.25} \fs{0.37} & \textbf{74.37} \fs{0.42} & \textbf{82.81} \fs{0.48} \\
    \midrule
    GradNorm ~\citep{chen2018gradnorm} & 67.35 \fs{0.15} & 73.53 \fs{0.23} & 81.03 \fs{0.71} \\
    GradNorm + PCGrad~\citep{yu2020gradient} & \textbf{67.83} \fs{0.19} & 73.91 \fs{0.09} & 82.72 \fs{0.19} \\
    GradNorm + MGDA~\cite{sener2018multi} & 36.99 \fs{2.11} & 57.94 \fs{0.92} & 59.12 \fs{0.63} \\
    GradNorm + CAGrad (Ours) & 67.53 \fs{0.26} & \textbf{74.72} \fs{0.19} & \textbf{83.15} \fs{0.56} \\
    \bottomrule
    \end{tabular}
    \vspace{5pt}
    \caption{Semi-supervised Learning with auxiliary tasks on CIFAR10. We report the average test accuracy over 3 independent runs for each method and mark the best result in bold.}
    \label{tab:ssl}
\end{table}
\vspace{-20pt}
\section{Conclusion}
\label{sec:conclusion}
In this work, we introduce the Conflict-Averse Gradient descent (CAGrad) algorithm that explicitly optimizes the minimum decrease rate of any specific task's loss while still provably converging to the optimum of the average loss. CAGrad generalizes the gradient descent and multiple gradient descent algorithm, and demonstrates improved performance across several challenging multi-task learning problems compared to the state-of-the-art methods. While we focus mainly on optimizing the average loss, an interesting future direction is to look at main objectives other than the average loss under the multi-task setting.


\newpage
\section*{Acknowledgements}
The research was conducted in the statistical learning and AI group (SLAI) and the Learning Agents
Research Group (LARG) in computer science at UT Austin. SLAI research is supported in part by CAREER-1846421, SenSE-2037267, EAGER-2041327, and Office of Navy Research, and NSF AI Institute for Foundations of Machine Learning (IFML). LARG research is supported in part
by NSF (CPS-1739964, IIS-1724157, FAIN-2019844), ONR (N00014-18-2243),
ARO (W911NF-19-2-0333), DARPA, Lockheed Martin, GM, Bosch, and UT
Austin's Good Systems grand challenge.  Peter Stone serves as the
Executive Director of Sony AI America and receives financial
compensation for this work.  The terms of this arrangement have been
reviewed and approved by the University of Texas at Austin in accordance
with its policy on objectivity in research.
Xingchao Liu is supported in part by a funding from BP.

\bibliography{main.bib}
\bibliographystyle{plain}
\section*{Checklist}
\begin{enumerate}

\item For all authors...
\begin{enumerate}
  \item Do the main claims made in the abstract and introduction accurately reflect the paper's contributions and scope?
    \answerYes{See Sec.~\ref{sec:convergence} for the convergence analysis, Fig.~\ref{fig:toy} for the challenges faced by previous methods, and Sec.~\ref{sec:experiment} for empirical evaluation of these challenges and the advantage of CAGrad.}
  \item Did you describe the limitations of your work?
    \answerYes{See Sec.~\ref{sec:conclusion}. Currently we mainly focus on optimizing the average loss, which could be replaced by other main objectives.}
  \item Did you discuss any potential negative societal impacts of your work?
    \answerNA{Our method does not have potential negative societal impacts.}
  \item Have you read the ethics review guidelines and ensured that your paper conforms to them?
    \answerYes{}
\end{enumerate}

\item If you are including theoretical results...
\begin{enumerate}
  \item Did you state the full set of assumptions of all theoretical results?
    \answerYes{The assumptions are stated in Thm.~\ref{thm:cagrad}.}
	\item Did you include complete proofs of all theoretical results?
    \answerYes{The complete proof is included in Appendix A.3.}
\end{enumerate}

\item If you ran experiments...
\begin{enumerate}
  \item Did you include the code, data, and instructions needed to reproduce the main experimental results (either in the supplemental material or as a URL)?
    \answerYes{We mention most of the details to reproduce the result in Sec.~\ref{sec:experiment} and provide the rest of details of each experiment in Appendix.B. The code comes with the supplementary material.}
  \item Did you specify all the training details (e.g., data splits, hyperparameters, how they were chosen)?
    \answerYes{See Appendix.B and Sec.~\ref{sec:experiment}.}
	\item Did you report error bars (e.g., with respect to the random seed after running experiments multiple times)?
    \answerYes{For each experiment except for the toy (since there is no stochasticity), we run over multiple ($\geq 3$) seeds.}
	\item Did you include the total amount of compute and the type of resources used (e.g., type of GPUs, internal cluster, or cloud provider)?
    \answerYes{We explicitly compare the computational efficiency in Fig.~\ref{fig:nyu-comparison}. More details on the resources are provided in the corresponding sections in Appendix.B.}
\end{enumerate}

\item If you are using existing assets (e.g., code, data, models) or curating/releasing new assets...
\begin{enumerate}
  \item If your work uses existing assets, did you cite the creators?
    \answerYes{For most of the experiment, we follow the exact experiment setup and use the corresponding open-source code from previous works and have cited and compared against them.}
  \item Did you mention the license of the assets?
    \answerYes{All code and data are publicly available under MIT license}
  \item Did you include any new assets either in the supplemental material or as a URL?
    \answerNo{No new assets are introduced for our experiment. The only thing we modified is a shrinked LeNet, where the details are provided in Appendix.B.}
  \item Did you discuss whether and how consent was obtained from people whose data you're using/curating?
    \answerNA{}
  \item Did you discuss whether the data you are using/curating contains personally identifiable information or offensive content?
    \answerNA{The data we use are publicly available data that has been used by a lot of prior research. There should be no personally identifiable information or offensive content.}
\end{enumerate}

\item If you used crowdsourcing or conducted research with human subjects...
\begin{enumerate}
  \item Did you include the full text of instructions given to participants and screenshots, if applicable?
    \answerNA{No human subjects involved.}
  \item Did you describe any potential participant risks, with links to Institutional Review Board (IRB) approvals, if applicable?
     \answerNA{}
  \item Did you include the estimated hourly wage paid to participants and the total amount spent on participant compensation?
     \answerNA{}
\end{enumerate}
\end{enumerate}
\appendix
\section{Algorithm Details}
In this section, we first formally introduce the Multiple Gradient Descent Algorithm and the Projecting Conflicting Gradients method. Then we provide the full proof of Thm.~\ref{thm:cagrad}.

\subsection{Multiple Gradient Descent Algorithm (MGDA)}
\label{apx:mgda}
The Multiple Gradient Descent Algorithm (MGDA) explicitly optimizes towards a Pareto-optimal point for multiple objectives (See the definition~\ref{def:pareto-optimality}). 
It is known that a necessary condition for $\theta$ to be a Pareto-optimal point is that we could find a convex combination of the task gradients at $\theta$ that results in the $0$ vector. Therefore, MGDA proposes to minimize the minimum possible convex combination of task gradients:
\begin{equation}
    \min \frac{1}{2}\norm{\sum_{i=1}^K w_i g_i}^2,~~\text{s.t.}~~\sum_{i=1}^K w_i = 1,~\text{and}~~\forall i, w_i \geq 0.
    \label{eq:mgda-dual}
\end{equation}
We call this the \emph{dual} objective for MGDA, as the primal objective of MGDA has a close connection to CAGrad's primal objective in Eq.~\eqref{eq:CAGrad}. Specifically, the \emph{primal} objective of MGDA is
\begin{equation}
    \max_{\norm{d}\leq 1} \min_i \langle d, g_i\rangle.
    \label{eq:mgda-primal}
\end{equation}
To see the primal-dual relationship, denote $g_w = \sum_i w_i g_i$, where $w \in \mathcal W \triangleq \{w \in \RR^K \colon ~~ \sum_i w_i = 1, ~ ~ w_i \geq 0, \forall i\in[K]\}.$ Note that $\min_i \langle g_i, d\rangle = \min_{w \in \mathcal W} \langle \sum_i w_i g_i, d\rangle$. The Lagrangian of Eq.~\eqref{eq:mgda-primal} is
\begin{equation}
\max_{d}\min_{\lambda \geq 0,w\in \mathcal W} \langle d, g_w\rangle - \frac{\lambda}{2} (\norm{d}^2 - 1).
\end{equation}
Since the problem is a convex programming and the Slater's condition holds when $c >0$ (On the other hand, if $c =0$, then it is easy to check that all the results hold trivially), the strong duality holds and we can exchange the min and max:
\begin{equation}
    \min_{\lambda \geq 0, w\in \mathcal W} \max_d \langle d, g_w\rangle - \frac{\lambda}{2} (\norm{d}^2 - 1).
\end{equation}
The optimal $d^* = g_w/\lambda$ and the resulting primal objective is therefore
\begin{equation}
    \min_{\lambda \geq 0, w\in \mathcal W} \lambda (\frac{1}{2}\norm{g_w}^2 + 1). 
\end{equation}
Here, $\lambda$ corresponds to the constraint $\norm{d} \leq 1$. If we fix $\lambda$ to be any constant, then we recover the dual objective in Eq.~\eqref{eq:mgda-dual}.
\paragraph{Remark} Looking at the primal form of MGDA in Eq.~\eqref{eq:mgda-primal}, the major difference between MGDA and CAGrad is that the new update vector $d$ is searched around the $0$ vector for MGDA and $g_0$ for CAGrad. Therefore, theoretically both MGDA and CAGrad optimizes the worst local update, but MGDA is more conservative and can converge to any point on the Pareto set without explicit control (See Thm. 2 from ~\citep{desideri2012multiple}). This also explains MGDA's behavior in practice that it often learns much slower than other methods.

\subsection{Projecting Conflicting Gradients (PCGrad)}
\label{apx:pcgrad}
Identifying that a major challenge for multi-task optimization is the conflicting gradient, Yu et al.~\citep{yu2020gradient} propose to project each task gradient to the normal plane of others before combining them together to form the final update vector. In the following, we provide the full algorithm of the Projecting Conflicting Gradients (PCGrad):

\begin{algorithm*}[h]
    \caption{Projecting Conflicting Gradient Update Rule}
    \begin{algorithmic} 
        \STATE \textbf{Input}: model parameter vector $\theta$ and differentiable loss functions $\{L_i\}_{i=1}^K$.
        \STATE $g_i \leftarrow \nabla_\theta L_i(\theta)$.
        \STATE $g^{\text{PC}}_i = g_i,~\forall i$.
        \FOR{task $i \in [K]$}
            \FOR{$j\neq i \in [K]$ in random order}
                \IF{$g_i^{\text{PC}}\cdot g_j < 0$}
                    \STATE $g_i^{\text{PC}} = g_i^{\text{PC}} - \frac{g_i^{\text{PC}}\cdot g_j}{\norm{g_j}^2}g_j$.
                \ENDIF 
            \ENDFOR 
        \ENDFOR
        \STATE\textbf{Return} the new update vector $d = g^\text{PC} = \frac{1}{K}\sum_i g_i^\text{PC}$.
    \end{algorithmic}
    \label{alg:pcgrad}
\end{algorithm*}

Fig.~\ref{fig:compare} provides a visualization of PCGrad's update rule for two-task learning (the 3rd column). Different from MGDA and CAGrad, PCGrad does not have a clear optimization objective at each step, which makes it hard to analyze PCGrad's convergence guarantee in general. In practice, the random ordering to do the projection is particularly important for PCGrad to work well~\citep{yu2020gradient}, which suggests that the intuition of removing the ``conflicting" part of each gradient might not be always correct. For the convergence analysis, Yu et al. establishes the convergence guarantee for PCGrad only under the two-task learning setting. Moreover, PCGrad is only guaranteed to converge to the Pareto set without explict control over which point it will arrive at (See Thm.~\ref{thm:pcgrad} in the following). 

\begin{thm}[Convergence of PCGrad~\citep{yu2020gradient}]
Consider two-task learning, assume the loss functions $L_1$ and $L_2$ are convex and differentiable. Suppose the gradient of $L_0 = (L_1+L_2)/2$ is $H$-Lipschitz with $H >0$. Then, the PCGrad update rule with step size $t\leq 1/H$ will converge to a Pareto-stationary point.
\label{thm:pcgrad}
\end{thm}

\subsection{Conflit-Averse Gradient descent (CAGrad)}
\label{apx:cagrad}
We provide the full derivation of CAGrad and the proof for its convergence in this section. Our proof assumes $L_0$ is a general function with gradient $g_0 = \nabla L_0$, that is, it does not have to be the average of $L_i$ as the case we focus on in the main paper.
\begin{lem}
Let $d^*$ be the solution of 
$$\max_{d\in \RR^m} \min_{i\in [K]} g_i\tt d ~~~s.t.~~~ \norm{g_0 - d} \leq c \norm{g_0}, 
$$
where $c \geq 0$, and $g_0, g_1,\ldots, g_K\in \RR^m$.
Then we have 
$$d^* = g_0 + \frac{c\norm{g_0}}{\norm{g_{w^*}}} g_{w^*},$$
where $g_{w^*} = \sum_i w^*_i g_i$ and $w^*$ is the solution of 
\bbb \label{equ:optW}
\min_{w\geq \mathcal W} g_w\tt g_0 + c\norm{g_0} \norm{g_w},
\eee 
where $\mathcal W = \{w \in \RR^K \colon ~~ \sum_i w_i = 1, ~ ~ w_i \geq 0, \forall i\in[K]\}.$
In addition, 
\bbb \label{equ:strongduality}
\min_i g_i \tt d^* =  g_{w^*}\tt g_0 + c\norm{g_0} \norm{g_{w^*}}. 
\eee  
\end{lem}
\begin{proof}
  Denote $\phi = c^2 \norm{g_0}^2$. 
 Note that $\min_i \langle g_i, d\rangle = \min_{w \in \mathcal W} \langle \sum_i w_i g_i, d\rangle$. 
 %
The Lagrangian of the objective in Eq.~\eqref{eq:CAGrad} is
$$
\max_{d\in \RR^m} \min_{\lambda \geq 0, w\in \mathcal W}  g_w ^\top d - \frac{\lambda}{2} (\norm{g_0 - d}^2 -  \phi).
$$
Since the problem is a convex programming and the Slater's condition holds when $c >0$ (On the other hand, if $c =0$, then it is easy to check that all the results hold trivially), the strong duality holds and we can exchange the min and max: 
$$
 \min_{\lambda \geq 0, w \in \mathcal W}  \max_{d\in \RR^m} g_w ^\top d - \frac{\lambda}{2} \norm{g_0 - d}^2 + \frac{\lambda  \phi}{2}.
$$
With  $\lambda,w$ fixing, 
the optimal $d$ is achieved when $d = g_0 + g_w/\lambda$,
yielding the following dual problem 
$$
 \min_{w, \lambda \geq 0}  g_w ^\top(g_0 + g_w/\lambda) - \frac{\lambda}{2} \norm{g_w/\lambda}^2 + \frac{\lambda }{2} \phi .
$$
This is equivalent to 
$$
 \min_{w, \lambda \geq 0}  g_w ^\top g_0    +  \frac{1}{2\lambda} \norm{g_w}^2+ \frac{\lambda  \phi }{2}.
$$
Optimizing out the $\lambda$ we have
$$
 \min_{w\in \mathcal W}  g_w ^\top g_0    +  \sqrt{\phi} \norm{g_w},
$$
where the optimal $\lambda = \norm{g_w}/\phi^{1/2}$.
This solves the problem.  
\eqref{equ:strongduality} is the consequence of the strong duality. 
\end{proof}
\paragraph{Convergence Analysis} 

\begin{ass} \label{ass:conditions}
Assume individual loss functions $L_0, L_1,\ldots, L_K$ are differentiable on $\RR^m$ and their gradients $\nabla L_i(\theta)$ are all $H$-Lipschitz, i.e. $\norm{\nabla L_i(x) - \nabla L_i(y)} \leq H\norm{x-y}$ for $i=0,1, \ldots, K$, where $H\in(0,\infty)$. 
Assume $L_0^* = \inf_{\theta\in \RR^m} L_0(\theta) > -\infty$. 
\end{ass} 
\begin{thm}[Convergence of CAGrad]
\label{thm:cagradappendix}
Assume Assumption \ref{ass:conditions} holds. 
With a fixed step size $\alpha$ satisfying $0 < \alpha \leq 1/H$, 
we have for the CAGrad in Alg.~\ref{alg:caged}: 

1) If $0 \leq c < 1$, then CAGrad converges to stationary points of $L_0$ 
convergence rate in that 
$$
\sum_{t=0}^T \norm{g_0(\theta_t)}^2
\leq \frac{2(L_0(\theta_0) - L_0^*)}{{\alpha} (1-c^2) }.
$$

2) For any $c \geq 0$, all the fixed point of CAGrad are  Pareto-stationary points of $(L_0,L_1, \ldots, L_K)$. 

\end{thm}

\begin{proof}
We will first prove 1). Consider the $t$-th optimization step and denote $d^*(\theta_t)$ the update direction obtained by solving 
\eqref{eq:CAGrad} at the $t$-th iteration. 
Then we have 
\bb 
L_{0}(\theta_{t+1}) - L_{0}(\theta_t) 
& = L_{0}(\theta_t - \alpha d^*(\theta_t)) - L_{0}(\theta_t)  \\
& \leq - \alpha g_0(\theta_t)^\top  d^*(\theta_t) + 
\frac{H \alpha^2}{2} \norm{d^*(\theta_t)}^2 \\
 & \leq  - \alpha g_0(\theta_t)^\top d^*(\theta_t) + 
\frac{\alpha}{2} \norm{d^*(\theta_t)}^2 \ant{$\alpha \leq 1/H$}\\ 
 & \leq  - \frac{\alpha}{2}
 \left(\norm{g_0(\theta_t}^2 + \norm{d^*(\theta_t)}^2 - \norm{g_0(\theta_t) - d^*(\theta_t)}^2 \right)  + 
\frac{\alpha}{2} \norm{d^*(\theta_t)}^2 \\ 
 & = - \frac{\alpha}{2} \left (
   \norm{g_0(\theta_t)}^2 - 
 \norm{d^*(\theta_t) - g_0(\theta_t)}^2\right ) \\
 & \leq -\frac{\alpha}{2}(1-c^2)\norm{g_0(\theta_t)}^2 ~~~\ant{by the constraint in \eqref{eq:CAGrad}}
\ee 
Using telescoping sums, we have $L_{0}(\theta_{T+1}) - L_{0}(0) = - (\alpha/2)(1-c^2) \sum_{t=0}^T \norm{g_0(\theta_t)}^2$. Therefore 
$$
\min_{t\leq T}\norm{g_0(\theta_t)}^2
\leq\frac{1}{T+1} \sum_{t=0}^{T}\norm{g_0(\theta_t)}^2
\leq  \frac{2(L_{0}(0) - L_{0}(\theta_{T+1}))}{{\alpha} (1-c^2) (T+1)}.
$$
Therefore, if $L_0$ is lower bounded, that is, $L_0^* := \inf_{\theta\in \RR^m} L_0(\theta) > -\infty$, then $\min_{t\leq T}\norm{g_0(\theta_t)}^2
  = O(1/T)$.

For general $c\geq 0$, in the fixed point, we have $d^*(\theta) =g_0(\theta) + \lambda g_{w^*}(\theta)=0$, which readily match the definition of Pareto Stationarity. 
\end{proof}

In the following, we show an additional result that when $c\geq 1$, and we use a properly decaying step size, 
the limit points of CAGrad are 
either stationary points of $L_0$, or Pareto-stationary points of $(L_1,\ldots, L_K)$. 
\begin{thm}
Under Assumption~\ref{ass:conditions}, assume $c\geq 1$ and we a time varying step size satisfying 
$$\alpha_t  \leq \frac{\norm{g_{w^*_t}(\theta_t)}}{H(c-1) \norm{g_0(\theta_t)}},$$ 
where $w^*_t$ is the solution of \eqref{equ:optW} at the $t$-th iteration, 
then we have 
\bb 
\sum_{t=0}^T \alpha_t \norm{g_0(\theta_t)}\norm{g_{w_t^*}(\theta_t)} \leq 2\frac{\min_{i}(L_i(\theta_0)- L_i(\theta_{T+1}))}{(c-1) }. 
\ee 
\end{thm}
Therefore, if we hae $L_i^* = \inf_{\theta\in\RR^m} L(\theta)>-\infty$ and $c > 1$,  then  we have 
$\alpha_t\norm{g_0(\theta_t)}\norm{g_{w_t^*}(\theta_t)} \to 0$ as $t\to \infty$, meaning that we have either $\alpha_t \to 0$, or $\norm{g_0(\theta_t)} \to 0$ or $\norm{g_{w_t^*}(\theta_t)} \to 0$.

In this case, the actual behavior of the algorithm depends on the specific choice of 
the step size. For example, if we take $ \alpha_t = \frac{\norm{g_{w^*_t}(\theta_t)}  }{H (c-1) \norm{g_0(\theta_t)}}$,  
then the result becomes 
$$
\sum_{t=0}^T 
\norm{g_{w_t^*}(\theta_t)}^2 \leq 2 H {\min_{i}(L_i(\theta_0)- L_i(\theta_{T+1}))}.
$$
which ensures $\norm{g_{w_t^*}(\theta_t)}^2 \to 0$.
\begin{proof} 
For any task $i\in[K]$, 
\bb 
L_i(\theta_{t+1}) - L_i(\theta) 
& \leq -\alpha_t g_i(\theta_t)^\top d^*(\theta_t) + \frac{H\alpha_t^2}{2} \norm{d^*(\theta_t)}^2 \\
& \leq -\alpha_t  \min_i g_i(\theta_t)^\top d^*(\theta_t) + \frac{H\alpha_t^2}{2} \norm{d^*(\theta_t)}^2 \\
& \leq -\alpha_t  
\left (g_{w^*_t}(\theta_t)  \tt g_0(\theta_t) + c \norm{g_0(\theta_t)} \norm{g_{w^*_t}(\theta_t)} \right )
+ \frac{H\alpha_t^2}{2} \norm{d^*(\theta_t)}^2 \ant{by \eqref{equ:strongduality}}\\
\ee 
Meanwhile, note that 
\bb 
\norm{d^*(\theta_t)}^2
& =\norm{g_0(\theta_t) +  \frac{c\norm{g_0(\theta_t)}}{\norm{g_{w^*_t}(\theta_t)}}{g_{w^*_t}(\theta_t)}}^2 \\
& 
=(c^2 + 1)\norm{g_0(\theta_t)}^2 +  
2\frac{c\norm{g_0(\theta_t)}}{\norm{g_{w^*_t}(\theta_t)}} 
g_0(\theta_t)\tt {g_{w^*_t}(\theta_t)} \\
& 
=2c \frac{\norm{g_0(\theta_t)}}{\norm{g_{w^*_t}(\theta_t)}} 
\left ( g_{w^*_t}(\theta_t)  \tt g_0(\theta_t) + c \norm{g_0(\theta_t)} \norm{g_{w^*_t}(\theta_t)} \right ) 
+ (1-c^2) \norm{g_0(\theta_t)}^2. 
\ee
Therefore, 
\bb 
& L_i(\theta_{t+1}) - L_i(\theta)  \\
& \leq -
 \alpha_t \left(1 - {H\alpha_t} c \frac{\norm{g_0(\theta_t)}}{\norm{g_{w^*_t}(\theta_t)}} \right  )
\left (g_{w^*_t}(\theta_t)  \tt g_0(\theta_t) + c \norm{g_0(\theta_t)} \norm{g_{w^*_t}(\theta_t)} \right ) 
+  \frac{H\alpha_t^2}{2}(c^2-1) \norm{g_0(\theta_t)}^2 \\
& \overset{(*)}{\leq} -
 \alpha_t \left(1 - {H\alpha_t} c \frac{\norm{g_0(\theta_t)}}{\norm{g_{w^*_t}(\theta_t)}} \right  )
 (c-1) \norm{g_0(\theta_t)} \norm{g_{w^*_t}(\theta_t)}  
-  \frac{H\alpha_t^2}{2}(c^2-1) \norm{g_0(\theta_t)}^2 \\ 
&
= - \alpha_t (c-1)  
\norm{g_0(\theta_t)} \norm{g_{w^*_t}(\theta_t)} 
  + \frac{H\alpha_t^2}{2}(c-1)^2  \norm{g_0(\theta_t)}^2 \\ 
 & \leq  -\frac{1}{2} \alpha_t (c-1)  
\norm{g_0(\theta_t)} \norm{g_{w^*_t}(\theta_t)}  \ant{assume $\alpha_{\red{t}} \leq  \frac{\norm{g_{w^*_t}(\theta_t)}  }{H (c-1) \norm{g_0(\theta_t)}},$ $c\geq 1$}
\ee 
where inequality (*) uses Cauchy-Schwarz inequality. 
Therefore, a telescoping sum gives 
$$
\sum_{t=0}^T \alpha_t \norm{g_0(\theta_t)}\norm{g_{w_t^*}(\theta_t)} \leq 2\frac{\min_{i}(L_i(\theta_0)- L_i(\theta_{T+1}))}{(c-1) },
$$
when $c\geq 1$. 
\end{proof}

\section{Experiment Details}
\label{apx:exp}

\subsection{Multi-Fashion+MNIST}
\label{apx:exp_mnist}
\paragraph{Experiment Details~~~} We follow the experiment setup from~\citep{mahapatra2020multi} and use the same shrinked LeNet that consists of the following layers as the shared base network: \textsc{Conv(1,5,9,1)}, \textsc{MaxPool2d(2)}, \textsc{ReLU}, \textsc{BatchNorm2d(5)}, \textsc{Conv2d(5,10,5,1)}, \textsc{MaxPool2d(2)}, \textsc{ReLU}, \textsc{BatchNorm1d(250)}, \textsc{Linear(250, 50)}. Then a task-specific linear head \textsc{Linear(50, 10)} is attached to the shared base for the MNIST and FashionMNIST prediction. We use Adam~\citep{kingma2014adam} optimizer with a 0.001 learning rate and 0.01 weight decay, and then train for 50 epochs with a batch size of 256. The training set consists of 120000 images of size 36x36 and the test set consists of 20000 images of the same size.

\subsection{Multi-task Supervised Learning}
\label{apx:exp_sl}
\paragraph{Experiment Details~~~} For the multi-task supervised learning experiments on the NYU-v2 and CityScapes datasets, we follow exactly the same setup from MTAN~\citep{liu2019end}. We describe the details in the following. We adopt the SegNet~\citep{badrinarayanan2017segnet} architecture as the backbone network and apply the attention mechanism from MTAN~\citep{liu2019end} on top of it. For the CityScapes dataset, we use the 7-class semantics labels. We train MTAN, Cross-Stitch, PCGrad and CAGrad with 200 epochs with a batch size of 2 for NYU-v2 and a batch size of 8 for CityScapes, using the Adam~\citep{kingma2014adam} optimizer with a learning rate of 0.0001. We further decay the learning rate to 0.00005 at the 100th epoch. As Liu et al. do not separately create a validation set, they average the test performance of each method in the last 10 epochs. We follow this and also average the test performance over the last 10 epochs, but additionally run over 3 seeds and calculate the mean and the standard error. We train CAGrad with $c \in \{0.1, 0.2, 0.3, 0.4, 0.5, 0.6, 0.7, 0.8, 0.9\}$ and pick the best $c$ using their corresponding averaged training performance ($c=0.4$ for NYU-v2 and $c=0.4$ for CityScapes).

We also provide the final test losses and the per-epoch training times of each method in Fig.~\ref{fig:nyu-comparison}. 

\begin{figure*}[ht]
    \centering
    \includegraphics[width=\textwidth]{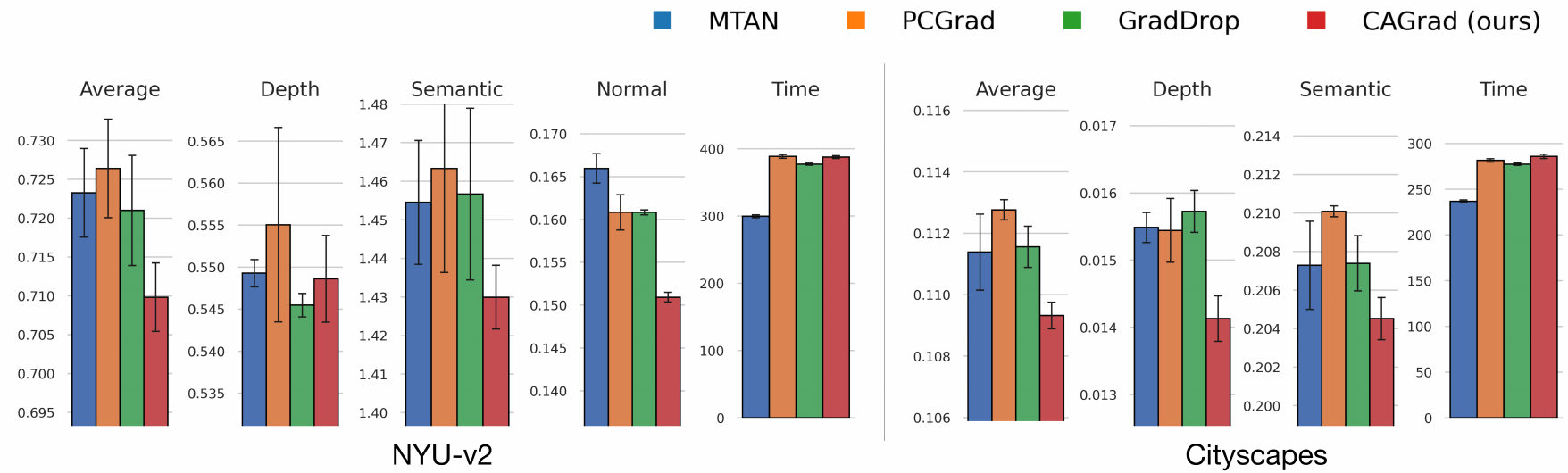}
    \vspace{-15pt}
    \caption{Test loss and training time comparison on NYU-v2 and Cityscapes.}
    \label{fig:nyu-comparison}
    \vspace{-5pt}
\end{figure*}

\paragraph{More Ablation Studies on NYU-v2 and CityScapes Datasets} We conduct the following additional studies on NYU-v2 and CityScapes datasets: 1) How do different methods perform when we additional apply the uncertain weight method~\citep{kendall2018multi}? 2) How do CAGrad perform with different values of $c$? 3) How does PCGrad perform when we enlarge the learning rate? Specifically we double the learning rate to 2e-4. Results are provided in Tab.~\ref{tab:nyu-v2-abl} and Tab.~\ref{tab:cityscapes-abl}. We can see that CAGrad perform consistently with different values of $0 < c < 1$. PCGrad with larger learning rate will not perform better. Under the uncertain weights, MTAN and PCGrad indeed perform better but CAGrad is still comparable or better than them.

\begin{table*}[t]
    \centering
    \resizebox{\textwidth}{!}{%
    \begin{tabular}{ccrrrrrrrrrc}
    \toprule
     & &  \multicolumn{2}{c}{Segmentation} & \multicolumn{2}{c}{Depth} & \multicolumn{5}{c}{Surface Normal} & \\
    \cmidrule(lr){3-4}\cmidrule(lr){5-6}\cmidrule(lr){7-11}
     \#P. & Method &  \multicolumn{2}{c}{(Higher Better)} &  \multicolumn{2}{c}{(Lower Better)}  & \multicolumn{2}{c}{\makecell{Angle Distance\\(Lower Better)}} & \multicolumn{3}{c}{\makecell{Within $t^\circ$\\(Higher Better)}} & $\Delta m \% \downarrow$\\
     & & mIoU  & Pix Acc  & Abs Err & Rel Err & Mean & Median & 11.25 & 22.5 & 30 & \\
     \midrule
     3 & Independent & 38.30 & 63.76 & 0.6754 & 0.2780 & 25.01 & 19.21 & 30.14 & 57.20 & 69.15 & \\
     \midrule
     $\approx$3 & Cross-Stitch~\citep{misra2016cross} & 37.42 & 63.51 & 0.5487 & 0.2188 & 28.85 & 24.52 & 22.75 & 46.58 & 59.56 & 6.96\\
    [0.15cm]
     $1.77$ & MTAN~\citep{liu2019end} & 39.29 & 65.33 & 0.5493 & 0.2263 & 28.15 & 23.96 & 22.09 & 47.50 & 61.08 & 5.59\\
    [0.15cm]
     $1.77$ & MGDA~\citep{sener2018multi} & 30.47 & 59.90 & 0.6070 & 0.2555 & 24.88 & 19.45 & 29.18 & 56.88 & 69.36 & 1.38\\
    [0.15cm]
     $1.77$ & PCGrad~\citep{yu2020gradient} (lr=1e-4) & 38.06 & 64.64 & 0.5550 & 0.2325 & 27.41 & 22.80 & 23.86 & 49.83 & 63.14 & 3.97\\
    [0.15cm]
     $1.77$ & PCGrad~\citep{yu2020gradient} (lr=2e-4) & 37.70 & 63.40 & 0.5871 & 0.2482 & 28.18 & 24.09 & 21.94 & 47.20 & 60.87 & 8.12\\
    [0.15cm]
     $1.77$ & GradDrop~\citep{chen2020just} & 39.39 & 65.12 & 0.5455 & 0.2279 & 27.48 & 22.96 & 23.38 & 49.44 & 62.87 & 3.58\\
    \midrule
     $1.77$ & CAGrad ($c$=0.2) & 39.15 & 65.45 & 0.5563 & 0.2295 & 26.74 & 21.93 & 25.17 & 51.55 & 64.70 & 1.55\\
    [0.15cm]
     $1.77$ & CAGrad ($c$=0.4) & 39.79 & 65.49 & 0.5486 & 0.2250 & 26.31 & 21.58 & 25.61 & 52.36 & 65.58 & 0.20\\
     [0.15cm]
     $1.77$ & CAGrad ($c$=0.6) & 39.54 & 65.60 & 0.5340 & 0.2199 & 25.87 & 20.94 & 25.88 & 53.78 & 67.00 & -1.36\\
    [0.15cm]
     $1.77$ & CAGrad ($c$=0.8) & 39.18 & 64.97 & 0.5379 & 0.2229 & 25.42 & 20.47 & 27.37 & 54.73 & 67.73 & -2.29\\
    \midrule
     $1.77$ & MTAN~\citep{liu2019end} (Uncert. Weights) & 38.74 & 64.70 & 0.5360 & 0.2243 & 26.52 & 21.71 & 25.50 & 52.02 & 65.14 & 0.75\\
    [0.15cm]
     $1.77$ & PCGrad~\citep{yu2020gradient} (Uncert. Weights) & 37.81 & 64.35 & 0.5318 & 0.2242 & 26.53 & 21.73 & 25.45 & 51.98 & 65.16 & 1.04\\
    [0.15cm]
     $1.77$ & CAGrad ($c$=0.2) (Uncert. Weights) & 38.87 & 65.19 & 0.5357 & 0.2227 & 26.38 & 21.64 & 25.66 & 52.21 & 65.39 & 0.319\\
    [0.15cm]
     $1.77$ & CAGrad ($c$=0.4) (Uncert. Weights) & 38.89 & 64.98 & 0.5313 & 0.2242 & 25.71 & 20.72 & 26.89 & 54.14 & 67.13 & -1.59\\
     [0.15cm]
     $1.77$ & CAGrad ($c$=0.6) (Uncert. Weights) & 39.80 & 65.32 & 0.5334 & 0.2242 & 25.69 & 20.91 & 26.89 & 54.14 & 67.13 & -1.59\\
    [0.15cm]
     $1.77$ & CAGrad ($c$=0.8) (Uncert. Weights) & 39.20 & 65.15 & 0.5322 & 0.2202 & 25.28 & 20.17 & 27.83 & 55.41 & 68.25 & -3.14\\
    \bottomrule 
    \end{tabular}
    }
    \caption{Multi-task learning results on NYU-v2 dataset. \#P denotes the relative model size compared to the vanilla SegNet. Each experiment is repeated over 3 random seeds and the mean is reported.}
    \label{tab:nyu-v2-abl}
\end{table*}

\begin{table*}[ht]
    \centering
    \resizebox{0.65\textwidth}{!}{%
    \begin{tabular}{ccrrrrc}
    \toprule
     & & \multicolumn{2}{c}{Segmentation} & \multicolumn{2}{c}{Depth} & \\
    \cmidrule(lr){3-4}\cmidrule(lr){5-6}
     \#P. & Method &  \multicolumn{2}{c}{(Higher Better)} &  \multicolumn{2}{c}{(Lower Better)} & $\Delta m \% \downarrow$ \\
     & & mIoU  & Pix Acc  & Abs Err & Rel Err & \\
     \midrule
    2 & Independent & 74.01 & 93.16 & 0.0125 & 27.77 & \\
     \midrule
    $\approx$3 & Cross-Stitch~\citep{misra2016cross} & 73.08 & 92.79 & 0.0165 & 118.5 & 90.02\\
    [0.15cm]
    $1.77$ & MTAN~\citep{liu2019end}  & 75.18 & 93.49 & 0.0155 & 46.77 & 22.60\\
    [0.15cm]
    $1.77$ & MGDA~\citep{sener2018multi}  & 68.84 & 91.54 & 0.0309 & 33.50 & 44.14\\
    [0.15cm]
    $1.77$ & PCGrad~\citep{yu2020gradient} & 75.13 & 93.48 & 0.0154 & 42.07  & 18.29\\
    [0.15cm]
    $1.77$ & GradDrop~\citep{chen2020just}  & 75.27 & 93.53 & 0.0157 & 47.54 & 23.73\\
    \midrule 
    $1.77$ & CAGrad ($c$=0.2) & 75.18 & 93.49 & 0.0140 & 40.12 & 13.69 \\
    [0.15cm]
    $1.77$ & CAGrad ($c$=0.4) & 75.16 & 93.48 & 0.0141 & 37.60 & 11.64 \\
    [0.15cm]
    $1.77$ & CAGrad ($c$=0.6) & 74.31 & 93.39 & 0.0151 & 34.84 & 11.46 \\
    [0.15cm]
    $1.77$ & CAGrad ($c$=0.8) & 74.95 & 93.50 & 0.0143 & 36.05 & 10.74 \\
    \midrule
    $1.77$ & MTAN~\citep{liu2019end} (Uncert. Weights) & 75.02 & 93.36 & 0.0139 & 35.56 & 9.48\\
    [0.15cm]
    $1.77$ & PCGrad~\citep{yu2020gradient} (Uncert. Weights)& 74.68 & 93.36 & 0.0135 & 34.00  & 7.26\\
    [0.15cm]
    $1.77$ & CAGrad ($c$=0.2) (Uncert. Weights) & 75.05 & 93.45 & 0.0140 & 34.33 & 8.40 \\
    [0.15cm]
    $1.77$ & CAGrad ($c$=0.4) (Uncert. Weights) & 74.90 & 93.46 & 0.0141 & 34.84 & 9.13 \\
    [0.15cm]
    $1.77$ & CAGrad ($c$=0.6) (Uncert. Weights) & 74.89 & 93.45 & 0.0136 & 35.17 & 8.48 \\
    [0.15cm]
    $1.77$ & CAGrad ($c$=0.8) (Uncert. Weights) & 75.38 & 93.48 & 0.0141 & 35.54 & 9.63 \\
    \bottomrule
    \end{tabular}
    }
    \caption{Multi-task learning results on CityScapes Challenge. Each experiment is repeated over 3 random seeds and the mean is reported.}
    \label{tab:cityscapes-abl}
    \vspace{-5pt}
\end{table*}

\subsection{Multi-task Reinforcement Learning}
\label{apx:exp_rl}
\paragraph{Experiment Details~~~} The multi-task reinforcement learning experiments follow the exact setup from CARE~\citep{sodhani2021multi}. Specifically, it is built on top of the MTRL codebase~\citep{Sodhani2021MTRL}. We consider the MT10 and MT50 benchmarks from the MetaWorld environment~\citep{yu2020meta}. A visualization of the 50 tasks from MT50 is provided in Fig.~\ref{fig:mt50}.
\begin{figure}[ht]
    \centering
    \includegraphics[width=0.8\textwidth]{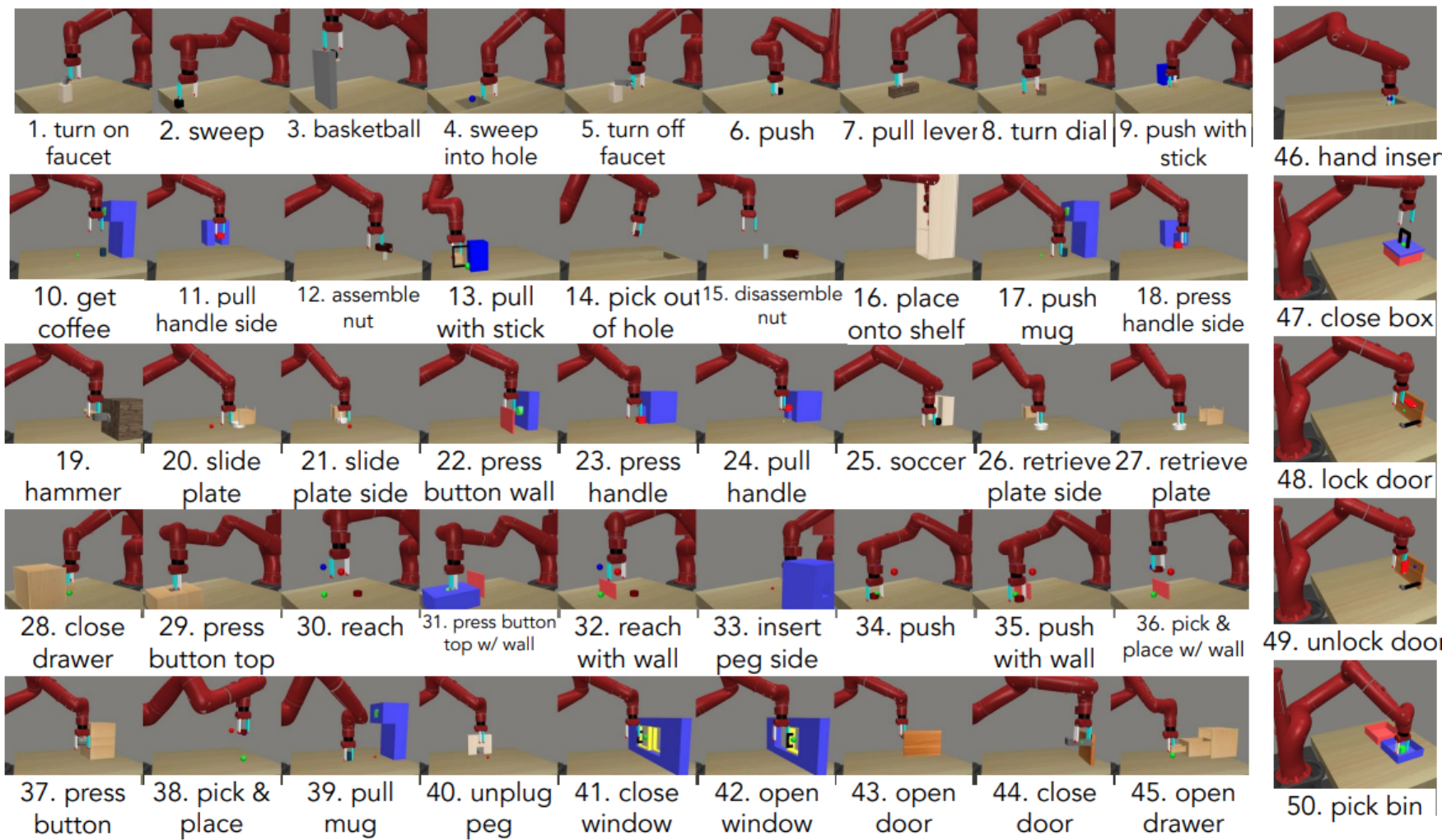}
    \caption{The 50 tasks in MT50 benchmark~\citep{yu2020meta}.}
    \label{fig:mt50}
\end{figure}
The MT10 benchmark consists of a subset of 10 tasks from the MT50 task pool. For all methods, we use Soft Actor Critic (SAC)~\citep{haarnoja2018soft} as the underlying reinforcement learning algorithm. All methods are trained over 2 million steps with a batch size of 1280. Following CARE~\citep{Sodhani2021MTRL}, we evaluate each method once every 10000 steps, and report the highest average test performance of a method over 10 random seeds over the entire training stage. For CAGrad-Fast, we sub-sample 4 and 8 tasks randomly at each optimization step as the $S$ (See Eq.~\eqref{eq:CAGrad-fast}) for the MT10 and MT50 experiments. For CAGrad, since MT10 and MT50 have 10 and 50 tasks, much more than the number of tasks in supervised MTL, so instead of using standard optimization library to solve the CAGrad objective, we apply 20 gradient descent steps to approximately solve the objective. The gradient descent is performed with a learning rate of 25 for MT10 and 50 for MT50, with a momentum of 0.5. We search the best $c$ from $\{0.1, 0.5, 0.9\}$ for MT10 and MT50 ($c=0.9$ for MT10 and $c=0.5$ for MT50). The computation efficiency is compared in Tab.~\ref{tab:mt-time}.\\
\begin{table}[htbp]
    \centering
    \begin{tabular}{lrr}
    \toprule
    Method & MT10 Time (sec) & MT50 Time (sec)\\
    \midrule
    PCGrad & 9.7 & 59.8\\
    CAGrad & 10.3 & 27.8\\
    CAGrad-Fast & \textbf{4.8} & \textbf{11.4} \\
    \bottomrule
    \end{tabular}
    \vspace{5pt}
    \caption{The training time per update step for PCGrad, CAGrad and CAGrad-Fast on MT10/50.}
    \vspace{-10pt}
    \label{tab:mt-time}
\end{table}
In principle, PCGrad should have the same time complexity as CAGrad. However, in practice, PCGrad projects the gradients following a random ordering of the tasks in a sequential fashion (See Alg.~\ref{alg:pcgrad}), so it requires a for loop over that task ordering, which makes it slow for a large number of tasks. Combined with the results from Tab.~\ref{tab:metaworld}, we see that CAGrad-Fast achieves comparable or better results than PCGrad with a roughly \textcolor{red}{2x} and \textcolor{red}{5x} speedup on MT10 and MT50.

\subsection{Semi-Supervised Learning with Auxiliary Tasks}
\label{apx:exp_ssl}
\paragraph{Experiment Details~~~}We provide the hyperparameters for reproducing the experiments in our main text. All the methods are applied upon the original ARML baseline, with the same configuration in~\citep{shi2020auxiliary}. Specifically, the batch size is $256$ and the optimizer is Adam. The learning rate is initialized to $0.005$ in the first $160,000$ iterations and decay to $0.001$ in the rest iterations. The backbone networks is a $\texttt{WRN-28-2}$ model. To stablize the training process, the features are extracted by a moving-averaged model like in ~\citep{tarvainen2017mean} with a moving-average factor of $0.95$. For PCGrad and MGDA, we use their official implementation without any change. For CAGrad (our method), we fix $c=0.1$ in all the experiments. The labeled images are randomly selected from the whole training set, and we repeat the experiments for 3 times on the same set of labeled images. We report the test accuracy of the model with the highest validation accuracy.

\paragraph{Training Losses~~~}We analyze the training losses of different methods to demonstrate the difference between these optimization methods. We report the losses, $L_{CE}$, $L_{aux}^1$ and $L_{aux}^2$, of the last epoch, when the number of labeled images is $2,000$. The losses are listed in Tab.~\ref{tab:semi_loss}. We have two key observations: (1) MGDA totally ignores the main task $L_{CE}$, yet it has the smallest loss on the second auxiliary task $L_{aux}^2$. This implies MGDA finds a sub-optimal solution on the Pareto front. (2) PCGrad and CAGrad can both decrease the averaged loss $L_0$ compared with the baseline ARML, however, CAGrad yields a smaller $L_0$ than PCGrad. 

\begin{table}[htbp]
    \centering
    \begin{tabular}{lcccc}
    \toprule
    Method & $L_{CE}$ & $L_{aux}^1$ & $L_{aux}^2$ & $L_0$\\
    \midrule
    ARML~\citep{shi2020auxiliary} & \textbf{0.0} \fs{0.0}  & 0.0574 \fs{0.0036} & -0.4946 \fs{0.0010} & -0.4372 \fs{0.0046} \\
    ARML + PCGrad~\citep{yu2020gradient} & \textbf{0.0} \fs{0.0}  & 0.0494 \fs{0.0088} & -0.4943 \fs{0.0007} & -0.4449 \fs{0.0095} \\
    ARML + MGDA~\cite{sener2018multi} & 0.407 \fs{0.018}  & 0.0453 \fs{0.0049}  & \textbf{-0.4980} \fs{0.0007} & -0.0463 \fs{0.0233}\\
    ARML + CAGrad (Ours) & \textbf{0.0} \fs{0.0}  & \textbf{0.0419} \fs{0.0034} & -0.4926 \fs{0.0023} & \textbf{-0.4507} \fs{0.0058} \\ 
    \bottomrule
    \end{tabular}
    \vspace{5pt}
    \caption{The Training Losses in the Last Epoch when the number of the labeled images is $2,000$. Values that are smaller than $10^{-6}$ are replaced by $0$. We report the averaged losses over 3 independent runs for each method, and mark the smallest losses in bold.}
    \label{tab:semi_loss}
\end{table}

\end{document}